\pdfoutput=1
\documentclass[letterpaper, 10 pt, conference]{ieeeconf}  

\IEEEoverridecommandlockouts                              

\usepackage{amsmath,amssymb}
\usepackage{ntheorem}
\usepackage{graphicx}
\usepackage{color} 
\usepackage{subfigure}
\usepackage{algpseudocode}
\usepackage{paralist}
\usepackage{nicefrac}
\usepackage{cite}
\usepackage{xcolor}

\newcommand\scalemath[2]{\scalebox{#1}{\mbox{\ensuremath{\displaystyle #2}}}}

\DeclareMathOperator{\diag}{diag}

\newtheorem{assumption}{Assumption}
\newtheorem{proposition}{Proposition}

\newtheorem{defn}{Definition}

\graphicspath{{figures/},{figures/tank/}}

\title{\LARGE \bf
Admittance Control Parameter Adaptation for Physical Human-Robot Interaction
}

\author{Chiara Talignani Landi, Federica Ferraguti, Lorenzo Sabattini, Cristian Secchi and Cesare Fantuzzi
  \thanks{Authors are
    with the Department of Sciences and Methods for Engineering,
    University of Modena and Reggio Emilia, Italy.  {\tt\small
      \{chiara.talignanilandi, federica.ferraguti, lorenzo.sabattini, cristian.secchi, cesare.fantuzzi\}@unimore.it}}%
}

\begin{document}

\maketitle
\thispagestyle{empty}
\pagestyle{empty}

\begin{abstract}
In physical human-robot interaction, the coexistence of robots and humans in the same workspace requires the guarantee of a stable interaction, trying to minimize the effort for the operator. To this aim, the admittance control is widely used and the appropriate selection of the its parameters is crucial, since they affect both the stability and the ability of the robot to interact with the user.
In this paper, we present a strategy for detecting deviations from the nominal behavior of an admittance-controlled robot and for adapting the parameters of the controller while guaranteeing the passivity. The proposed methodology is validated on a KUKA LWR 4+.     
\end{abstract}

\section{INTRODUCTION}
\label{introduction}
In physical human-robot interaction (pHRI) tasks, robots are designed to coexist and cooperate with humans in different applications within the same workspace, hence main concerns are related to safety and dependability. To obtain a compliant behavior on industrial robots, admittance control \cite{handbook} is typically utilized, since they are characterized by a stiff and non-backdrivable mechanical structure.
For example, admittance control has been exploited in \cite{iecon,hfr} to develop a walk-through programming technique on a position-controlled robot. The appropriate selection of the admittance parameters is crucial, since they define the way the robot interacts with the user.
Several researches have been performed to investigate the stability of a robot under admittance control during the interaction with a human operator. In \cite{peer} a stability analysis shows a strong dependency on the human arm impedance.
In particular, instability occurs when the human operator grasps the tool attached at the robot end-effector in a very stiff manner. Moreover, it was shown that it is always possible to find values for the inertia and damping parameters that lead the system to instability.
In \cite{eppinger,colgate}, the reason for instability was identified in the non-colocation principle.
All the identified instability issues produce, as a consequence, a deviation of the robot behavior from the desired one, imposed through the admittance control. This deviation results in high amplitude oscillations of the end-effector, that render the interaction with the robot unsafe for the user. To reduce them and restore the stability of the system the adaptation of the admittance parameters is usually performed.
In \cite{Gallagher,Lamy,Podobnik} the control gains and an estimate of the operator's stiffness are obtained using additional devices as EMG and force/torque sensors.
In \cite{tsumugiwa2} the authors increase the damping coefficient proportionally to the estimated stiffness of the human arm. 
In \cite{dimeas} the authors augment the desired inertia keeping a low damping, exploiting a frequency domain analysis on external forces. \\
\indent This paper proposes a novel methodology for adapting online the parameters of admittance-controlled robot used in pHRI tasks. 
 To achieve this objective, we will first introduce a heuristic that allows to easily detect online the deviations of the robot behavior from the nominal one. Then we will present a method for adapting the parameters to restore the nominal behavior without excessively increasing the pHRI effort for the user. Passivity will be exploited for ensuring safety in pHRI: in fact, as is well known \cite{secchi_book}, ensuring passivity leads to guaranteeing stability of the robot interacting with a possibly unknown environment. As discussed in~\cite{dimeas}, passivity often leads to defining over-conservative gains, but less conservative solutions can be achieved exploiting energy tanks~\cite{secchi}. In~\cite{Fer15tro} they have been exploited for passively changing inertia and stiffness in admittance control but the strategy proposed is still over-conservative since it exploits the energy dissipated even for reproducing passive situations (e.g. constant inertia and constant stiffness). This unwise use of energy needs a sufficiently big damping to be implemented. 
In this paper we will define a novel parameter adaptation methodology, that extracts energy from the tank only when strictly necessary (i.e. during variation of the parameters). This solution allows to implement the parameters variation while keeping a low damper and reducing the physical human effort.
\section{DETECTION OF UNDESIRED SITUATIONS}
\label{problem}
%
Let us consider a $n$-degrees of freedom ($n$-DOFs) manipulator controlled by using the admittance control.
Given a desired interaction model, namely a dynamic relation between the motion of the robot and the force applied by the environment, and given the external force, the corresponding position of the robot \mbox{$x_{ref}\in\mathbb{R}^n$} is tracked by means of a low-level position controller. 
In particular, in this paper we want to address the case of an industrial robotic manipulator manually driven by the human operator, hence we don't specify a desired pose (i.e. the elastic part in the model).
Define now \mbox{$x_{ref}\in\mathbb{R}^n$} as {the reference position computed by the admittance controller and $x=f(q)$ the pose of the end-effector, obtained from the joint positions $q\in\mathbb{R}^m$, $m \geq n$, through the forward kinematic map $f(\cdot)$}. We will hereafter make the following assumption:
	\begin{assumption}
		The low-level position controller is designed and tuned in such a way that the tracking error is negligible, namely \mbox{$x \simeq x_{ref}$}.
		\label{ass:pidideal}
	\end{assumption}
 Thus, under Assumption \ref{ass:pidideal}, we want to force the robot to interact with the environment according to a given mass-damper system characterized by:
\begin{equation}\label{eq:follower}
M_d \ddot{x} + D_d \dot{x} = F_{ext}
\end{equation} 
where $M_d$ and $D_d$ are the constant inertia and damping matrices, which are the parameters in the described control.
The external force $F_{ext}\in\mathbb{R}^n$ in \eqref{eq:follower} is assumed to be measured by a 6-DOF force/torque (F/T) sensor attached at the robot wrist flange.	
The controlled robot behaves as \eqref{eq:follower} and it is passive with respect to the pair $\left(F_{ext},\dot{x}\right)$. 

{We will hereafter introduce the following definition of \emph{nominal behavior}:
	\begin{defn}\label{def:nominal}\textbf{[Nominal behavior]}
		Consider the desired dynamic behavior for the admittance controlled robot~\eqref{eq:follower}. Then, under Assumption~\ref{ass:pidideal}, we define as \emph{nominal} the behavior of the robot when the following inequality holds
		\begin{equation}
		\label{eq:epsilon}
		\psi\left(\dot{x},\ddot{x},F_{ext}\right) =\Vert F_{ext}-M_d \ddot{x}-D_d\dot{x} \Vert \leq \varepsilon
		\end{equation}
		for an appropriately defined small $\varepsilon>0$. 
	\end{defn}
	}

For ease of notation, we will hereafter suppress the arguments of $\psi$. 
Since the acceleration is required, its value has to be measured using additional hardware (e.g. accelerometers) or estimated by means of filters (e.g. Kalman filter).

Since, during the execution of the cooperative task, the robot is coupled with a human operator, deviations from the nominal behavior could arise when he/she stiffens his/her arm. This causes oscillations of high amplitude that render the interaction unsafe for the user.
%
%
Considering Definition~\ref{def:nominal}, 
%
%
we propose to utilize~\eqref{eq:epsilon} as a heuristic for detecting deviations from the nominal behavior. Namely, when~\eqref{eq:epsilon} is not satisfied, we claim that the robot is deviating from the nominal behavior.
In order to detect when $\psi$ is over the threshold and the deviation from the nominal behavior occurs, avoiding false positives, a low pass filter has to be implemented.

\section{PASSIVE PARAMETER ADAPTATION}
\label{passivity}
{In this Section we will show how to adapt the parameter of the admittance control in order to restore the nominal behavior of the controlled robot in the presence of deviations, identified according to the procedure illustrated in Section~\ref{problem}. The proposed method will aim at minimizing the deviation of the admittance parameters from the desired interaction model. 

In the following, we will focus on adaptation of the parameters performed within a limited amount of time {and we will assume that, between two consecutive variations of inertia, a sufficient period of time elapses in which the inertia remains constant.} Namely, focusing on a single adaptation, we make the following assumption:
\begin{assumption}\label{ass:limitedtime}
	The variation of inertia and damping occurs in a limited time interval $\left[t_i,t_f\right]$. 
\end{assumption}	
	
	}

If the parameters have to be adapted, then the desired interaction model in \eqref{eq:follower} becomes the following variable admittance model:
\begin{equation}\label{eq:newadmittancemodel}
M_d(t) \ddot{x} + D_d(t) \dot{x} = F_{ext}
\end{equation}   
where $M_d(t)$ and $D_d(t)$ are the time-varying inertia and damping matrices.
In order to preserve their physical meaning, we assume that
$M_d(t)$ and $D_d(t)$ are symmetric and positive definite for all
$t\geq 0$. 
The main drawback due to the introduction of a variable interaction
model in an admittance control scheme is the loss of passivity of
the controlled robot \cite{Fer15tro}.

In this Section we will formulate a passivity framework for adapting the parameters while maintaining the passivity of the overall system. 
%
To guarantee the passivity of the system, the following proposition can be stated. 
\begin{proposition}
\label{prop:passivity1}
If 
\begin{equation}
\label{eq:cond1}
\dot{M}_d(t)-2D_d(t) \leq 0
\end{equation}
then the system in \eqref{eq:newadmittancemodel} is passive with respect to the
input-output pair $(F_{ext},\dot{x})$ and with storage function 
\begin{equation}
H(\dot{x})=\dfrac{1}{2}\dot{x}^TM_d(t)\dot{x}
\label{eq:storagenew}
\end{equation}
\end{proposition}
\begin{proof} With a slight abuse of notation, we will hereafter use $H(t)$ to indicate the value of $H(\dot{x})$ at time $t$. We have that:
\begin{equation}
  \label{eq:variablebalancepass}
\dot{H}(t) = \dot{x}^{T}M_d(t)\ddot{x}+\frac{1}{2}\dot{x}^{T}\dot{M}_d(t)\dot{x}
\end{equation}
Computing $\ddot{x}$ from \eqref{eq:newadmittancemodel} and replacing it in \eqref{eq:variablebalancepass} we obtain:
\begin{equation}
  \label{eq:variablebalance}
\dot{H}(t) = \dot{x}^TF_{ext}+\frac{1}{2}\dot{x}^T\left(\dot{M}_d(t)-2D_d(t)\right)\dot{x}
\end{equation}
whence
\begin{equation}
\resizebox{\hsize}{!}{$\displaystyle\int_{0}^{t} \dot{x}^T F_{ext} \, d \tau=H(t)-H(0)-\displaystyle\int_{0}^{t} \left[
  \frac{1}{2}\dot{x}^T\left(\dot{M}_d(t)-2D_d(t)\right)\dot{x}\right] \, d \tau$}
\end{equation}
Because of the variability of the inertia, in general the term between the brackets can be positive and the system can produce
energy. However, recalling that $H(t) \geq 0$, the condition \eqref{eq:cond1} allows to guarantee the passivity. Indeed:
\begin{equation}
\displaystyle\int_{0}^{t} \dot{x}^T F_{ext} \, d \tau \geq H(t)-H(0) \geq -H(0)
\label{eq:balance}
\end{equation}
\end{proof}
The condition \eqref{eq:cond1} is instantaneous and it guarantees that the effect due to the variation of the inertia, i.e. $\frac{1}{2}\dot{x}^T\left(\dot{M}_d(t)-2D_d(t)\right)\dot{x}$, is always dissipative. As shown in Prop. \ref{prop:passivity1}, this is sufficient for guaranteeing the passivity of the energy balance in \eqref{eq:balance}.

{Nevertheless, as pointed out in \cite{dimeas},} such a condition can be quite conservative since the passivity definition is based on energy balances rather than on power balances. Thus, if a sufficient amount of energy has been dissipated, some active behavior due to the inertia variation can be still allowed without violating the overall passivity of the system. 


In the following, we exploit energy tanks \cite{Fer15tro} for keeping track of the energy dissipated by the system and to develop a less conservative condition for guaranteeing passivity. This allows much more aggressive inertia variations at the price of a more complex control structure and some extra parameters to tune. Unlike \cite{Fer15tro}, the energy of the tank is used only in case some energy can be produced and not every time the desired inertia value is different from the nominal one.  

Then, we augment the interaction model \eqref{eq:newadmittancemodel} with an energy storing element, the tank, whose role is to store the energy dissipated by the controlled system. Formally, the augmented dynamics is given by:
\begin{equation}
\left\{
\begin{array}{l}
M_d(t) \ddot{x} + D_d(t) \dot{x} = F_{ext} \\
\displaystyle \dot{z}= \dfrac{\varphi}{z}P_D-\dfrac{\gamma}{z}P_{M}\\
\end{array}
\right.
\label{eq:newdynamic}
\end{equation}
where
\begin{equation}
  \label{eq:PdPm}
 \begin{array}{lr}
  P_D=\dot{x}^TD_d(t)\dot{x} \quad & \quad 
  P_M=\frac{1}{2}\dot{x}^T\dot{M}_d(t)\dot{x}
  \end{array}
\end{equation}
are the 
{dissipated power due to damping, and the dissipated/injected power due to the inertia variation}, respectively, and $z(t)\in\mathbb{R}$ is the state of the tank. Furthermore, let
\begin{equation}
\label{eq:tanken}
T(z)=\frac{1}{2}z^2
\end{equation}
be the energy stored in the tank. {With a slight abuse of notation, here and in the following we will use $T(t)$ to indicate the value of $T(z)$ at time $t$. We will hereafter assume that $\exists \delta,\bar{T}$, with $0<\delta<\bar{T}$, such that $\delta \leq T(t) \leq \bar{T}$.}
 The upper bound is guaranteed by the parameters $\varphi\in\{0,1\}$ and $\gamma \in\{0,1\}$ that disable the energy storage in case a maximum, application dependent, limit \mbox{$\bar{T} \in \mathbb{R}^+$} is reached. It is necessary to bound the available energy because, if there were no bounds, the energy could become very big as time increases and, even if the system keeps on being passive, it would be possible to implement practically unstable behaviors \cite{lee}. Then, we define:
 \begin{eqnarray}
 \varphi = \left\{
 \begin{array}{cl}
 1 & if\text{ }T(t)\leq \bar T \\
 0 & otherwise\\
 \end{array}
 \right.
 \label{eq:sigma}
 \gamma = \left\{
 \begin{array}{cl}
 \varphi & if\text{ }\dot{M}_d(t)\leq 0 \\
 1 & otherwise\\
 \end{array}
 \right.
 \label{eq:gamma}
 \end{eqnarray}
 
where $\varphi$ enables/disables the storage of dissipated energy and $\gamma$ enables/disables the injection \mbox{$\left(\dot{M}_d(t) \leq 0\right)$} of energy in the tank due to the inertia variation but it always allows to extract  \mbox{$\left(\dot{M}_d(t) > 0\right)$} energy from the tank.
The lower bound, required for avoiding singularities in \eqref{eq:newdynamic}, is guaranteed by carefully planning/forbidding the extraction of energy when $\delta>0$ is reached. Notice that the extraction of energy is due only to $P_M$. The tank initial state is
set to $z(0)$ such that \mbox{$T(z(0))>\delta$}.  

Furthermore, we have that
\begin{equation}
\label{dotT}
\dot{T}(t)=z \dot{z}=\varphi P_D-\gamma P_M
\end{equation}

{Define now $\lambda_j\left(\dot{M}_d(t)\right)$ as the $j$-th eigenvalue of $\dot{M}_d(t)$ at time $t$. We now introduce $\lambda_M$ as the maximum value for the eigenvalues of {$\dot{M}_d(t)$} over the time interval in which the parameter variation occurs, according to Assumption~\ref{ass:limitedtime}, namely:
	\begin{equation}
	\lambda_M = \max_{t\in\left[t_i,t_f\right]}\max_{j=1,\ldots,n}\lambda_j\left(\dot{M}_d(t)\right)
	\label{eq:lambdaM}
	\end{equation}
It is worth noting that this maximum is well defined, since the time interval is supposed to be bounded, {as well as the variation in the inertia matrix}.	
	}

Consider the following bound on the velocity
\begin{equation}
-\dot{x}_M \leq \dot{x} \leq \dot{x}_M
\label{eq:bound}
\end{equation}
Then, the following proposition holds. 

\begin{proposition}
\label{prop:passivity2}
If 
\begin{equation}
\label{eq:cond2}
\frac{1}{2}\lambda_{M}{\Vert \dot{x}_M \Vert}^2\left(t_f-t_i\right) \leq T(t_i)-\delta
\end{equation}
then the system in \eqref{eq:newdynamic} is passive with respect to the
input-output pair $(F_{ext},\dot{x})$ and with storage function 
\begin{equation}
W(\dot{x}(t),z(t))=H(\dot{x}(t))+T(z(t))
\label{eq:W}
\end{equation}
where $H(\dot{x}(t))$ and $T(z(t))$ are defined in
\eqref{eq:storagenew} and in \eqref{eq:tanken}.
\end{proposition}
\begin{proof} {With a slight abuse of notation, we will hereafter use $W(t)$ to indicate the value of $W(\dot{x}(t),z(t))$ at time $t$.} Since $H(t)$ and $T(t)$ are positive definite, $W(t)$ is positive definite too.
The time derivative $\dot{W}(t)$ can be computed as follows
\begin{equation}
\resizebox{\hsize}{!}{$
\begin{array}{ll}
\dot{W}(t)=\dot{H}(t)+\dot{T}(t)\!\!\!&=\dot{x}^T{F_{ext}}-\left(1-\varphi\right)P_D+\left(1-\gamma\right)P_M
\end{array}
$}
\label{eq:Wderivative}
\end{equation}
Since $\varphi\in\{0,1\}$ and $P_D \geq 0$, we have that
\begin{equation}
\dot{x}^T{F_{ext}} \geq \dot{H}(t)+\dot{T}(t)-\left(1-\gamma\right)P_M
\end{equation}
and
\begin{equation}
\resizebox{\hsize}{!}{$\displaystyle\int_{0}^{t} \dot{x}^T F_{ext} \, d \tau \geq H(t)-H(0)+T(t)-T(0)-\displaystyle\int_{0}^{t} \left(1-\gamma\right)P_M \, d \tau$}
\end{equation}
If \mbox{$\dot{M}_d(t) \leq 0$}, considering the fact that \mbox{$T(t) \geq 0$} and \mbox{$H(t) \geq 0$}, we have that
\begin{equation}
\displaystyle\int_{0}^{t} \dot{x}^T F_{ext} \, d \tau \geq -H(0)-T(0)
\end{equation}
namely, that the system is passive. Since we aim at changing the inertia from a constant value to another constant value, we have \mbox{$\dot{M}_d(t) > 0$} only for a specified limited interval $\left[t_i,t_f\right]$. Furthermore, from \eqref{eq:gamma}, we have that $\gamma=1$ and thus
\begin{equation}
\label{eq:passgar1}
\displaystyle\int_{0}^{t_f} \dot{x}^T F_{ext} \, d \tau \geq H(t_f)-H(0)+T(t_f)-T(0)
\end{equation}
We have no information on the energy that will be stored in the tank in the future, i.e. $T(t_f)$, but if we guarantee that
\begin{equation}
\label{eq:passgar}
T(t_f)-T(0) \geq 0
\end{equation}
from \eqref{eq:passgar1} and \eqref{eq:passgar} and recalling that $H(t_f) \geq 0 $, we can state the following passivity balance:
\begin{equation}
\label{eq:passbal}
\displaystyle\int_{0}^{t_f} \dot{x}^T F_{ext} \, d \tau \geq -H(0)
\end{equation}
The energy stored in the tank at the end of the interval is given by the sum of the energy currently available and of the dissipated energy and the energy due to the inertia variation in the interval of variation. {According to Assumption~\ref{ass:limitedtime}, it can be computed as follows:} 
\begin{equation}
\label{eq:Ttf}
T(t_f)=T(t_i)+\displaystyle\int_{t_i}^{t_f} P_D \, d \tau - \displaystyle\int_{t_i}^{t_f} P_M \, d \tau
\end{equation}
Thus, substituting \eqref{eq:Ttf} in \eqref{eq:passgar}, it follows that the passivity of the overall system is guaranteed if 
\begin{equation}
T(t_i)+\displaystyle\int_{t_i}^{t_f} P_D \, d \tau - \displaystyle\int_{t_i}^{t_f} P_M \, d \tau-T(0) \geq 0
\label{eq:part1}
\end{equation}
Since the term due to the dissipated power is always positive, to satisfy \eqref{eq:part1}, it is sufficient that
\begin{equation}
T(t_i)- \displaystyle\int_{t_i}^{t_f} P_M \, d \tau-T(0) \geq 0
\end{equation}
and thus
\begin{equation}
\label{eq:in}
\displaystyle\int_{t_i}^{t_f} P_M \, d \tau \leq T(t_i)- \delta
\end{equation}
Since $M_d(t)$ is symmetric, $\dot{M}_d(t)$ must be symmetric too. Considering the bound on the velocity \eqref{eq:bound}{, the defintion of $\lambda_M$ in~\eqref{eq:lambdaM}, } and reminding that \mbox{$\dot{M}_d(t) > 0$}, we can write
\begin{equation}
\displaystyle\int_{t_i}^{t_f} P_M \, d \tau \leq \frac{1}{2}\lambda_M{\Vert \dot{x} \Vert}^2\left(t_f-t_i\right)\leq \frac{1}{2}\lambda_M{\Vert \dot{x}_M \Vert}^2\left(t_f-t_i\right)
\end{equation}
and thus it follows that with condition~\eqref{eq:cond2}, namely
$$
\frac{1}{2}\lambda_M{\Vert \dot{x}_M \Vert}^2\left(t_f-t_i\right) \leq T(t_i)- \delta
$$
the inequalities \eqref{eq:in} and \eqref{eq:passgar} are satisfied and thus the passivity balance in \eqref{eq:passbal} is verified, which concludes the proof. 
\end{proof}
The condition \eqref{eq:cond2} requires to store the energy in the tank by updating a continuous dynamics but in the following we will show that it is more flexible than the one stated in Prop. \ref{prop:passivity1}. Since the desired inertia and damping are parameters that can be freely chosen, provided that they are symmetric and positive definite matrices, we will consider the following assumption, which is a common choice.
\begin{assumption}
\label{ass:diagonal}
The desired inertia and damping in \eqref{eq:newadmittancemodel} are diagonal matrices and they are defined as
\begin{equation}
\scalemath{0.9}{M_d(t)=\diag{m_1(t),\dots,m_n(t)};\quad
D_d(t)=\diag{d_1(t),\dots,d_n(t)}} \nonumber
\end{equation}
Since $M_d(t)$ is diagonal, $\dot{M}_d(t)$ is diagonal too and its eigenvalues are the elements on the diagonal. 
\end{assumption}
Under Assumption \ref{ass:diagonal}, we have that the general condition of passivity \eqref{eq:cond1} becomes 
\begin{equation}
\label{eq:newcond1}
\dot{m}_j(t) \leq 2d_j(t) \quad \forall j=1, \dots, n
\end{equation} 
{Furthermore, from~\eqref{eq:lambdaM} and from Assumption~\ref{ass:diagonal}, it follows that}
\begin{equation}
\dot{m}_j(t) \leq \lambda_M \quad \forall j=1, \dots, n
\end{equation}
and thus the condition \eqref{eq:cond2} obtained exploiting energy tanks becomes
\begin{equation}
\label{eq:newcond2}
\dot{m}_j(t) \leq \dfrac{2\left(T(t_i)-\delta\right)}{{\Vert \dot{x}_M \Vert}^2\left(t_f-t_i\right)} \quad \forall j=1, \dots, n
\end{equation}
From \eqref{eq:newcond1} and \eqref{eq:newcond2} it follows that the energy tank framework is less conservative than the general passivity approach every time the following inequality holds
\begin{equation}
\label{eq:bench}
\left(T(t_i)-\delta\right) > d_j{\Vert \dot{x}_M \Vert}^2\left(t_f-t_i\right) \quad \forall j=1, \dots, n
\end{equation}
which means that the energy available in the tank is larger than the maximum energy that could be dissipated in the time interval $\left[t_i,t_f\right]$. The inequality in \eqref{eq:bench} is usually verified and, in particular, in all the experiments we performed we never observed the opposite.

Thanks to the conditions stated in Props. \ref{prop:passivity1} and \ref{prop:passivity2}, the adaptation of the parameters to react to a deviation from the nominal behavior can be performed without violating the passivity of the system and thus guaranteeing the stability of the admittance-controlled robot.

\section{EXPERIMENTAL RESULTS}
\label{experiments}
We performed experimental tests on a KUKA LWR 4+\footnote{http://www.kuka-robotics.com/en} provided with a ATI Mini 45\footnote{http://www.ati-ia.com} 6 axis F/T sensor in order to
validate the theoretical findings presented in this paper. Let $\mathsf{\left(x,y,z,r_1,r_2,r_3\right)}$ be the reference frame of the robot end-effector. Due to space limitations, we will only report plots of the robot behavior along the $\mathsf{x}$-axis (DOFs are decoupled). 
The accompanying video clip shows the experiments described in the following.

\subsection{Detection of deviations from the nominal behavior}
Preliminary experiments have been performed to test the heuristic that we introduced in Section \ref{problem}, finding the threshold $\varepsilon$ that characterizes the nominal behavior.
Since the objective of this experiment was to evaluate the performance of the detection heuristic, we did not modify the inertia and damping parameters.
 In this way, we can easily show that, whenever the user interacts with the tool stiffening his/her arm, the forces of interaction oscillate and the system becomes unstable (Fig. \ref{fig:force1}). Thus, deviations from the nominal behavior occur, \eqref{eq:epsilon} is not satisfied and the deviation can be correctly detected. 
Under Assumption \ref{ass:diagonal}, the desired inertia and damping parameters were set as the following constant diagonal matrices:
\begin{equation}
		\scalemath{0.98}{ M_d=\diag\{2, 2, 2, 0.5, 0.5, 0.5\}\ kg \quad  
		 D_d=\diag\{D_{d1}, D_{d2}\}\ \nonumber}
\end{equation}
where $ D_{d1}=\diag\{30, 30, 30\}\ \nicefrac{Ns}{m} \nonumber $ and $ D_{d2}=\diag\{3, 3, 3\}\ \nicefrac{Nms}{rad} \nonumber $.
The interaction forces measured by the F/T sensor are reported in Fig. \ref{fig:force1}, while the evolution over time of $\psi$ as defined in \eqref{eq:epsilon} is shown in Fig. \ref{fig:epsilon}. The threshold in \eqref{eq:epsilon} was empirically chosen as $\varepsilon=10$. 
	\begin{figure}[tbp]
  		\centering
  		\includegraphics[width=\columnwidth]{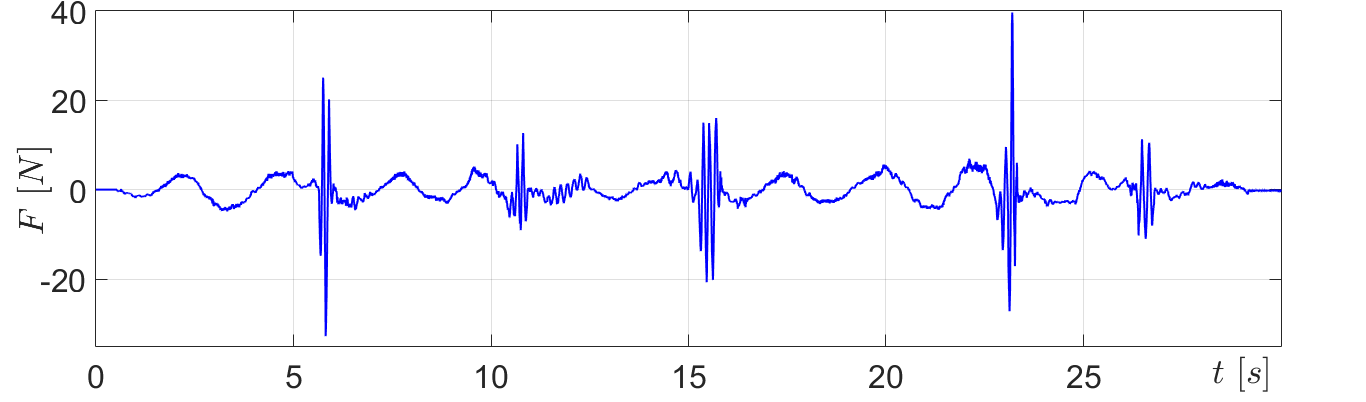}
  		\caption{Force along the $\mathsf{x}$-axis measured by the F/T sensor.}
 		\label{fig:force1}
	\end{figure}
	\begin{figure}[tbp]
 		 \centering
  		\includegraphics[width=\columnwidth, height=3.6cm]{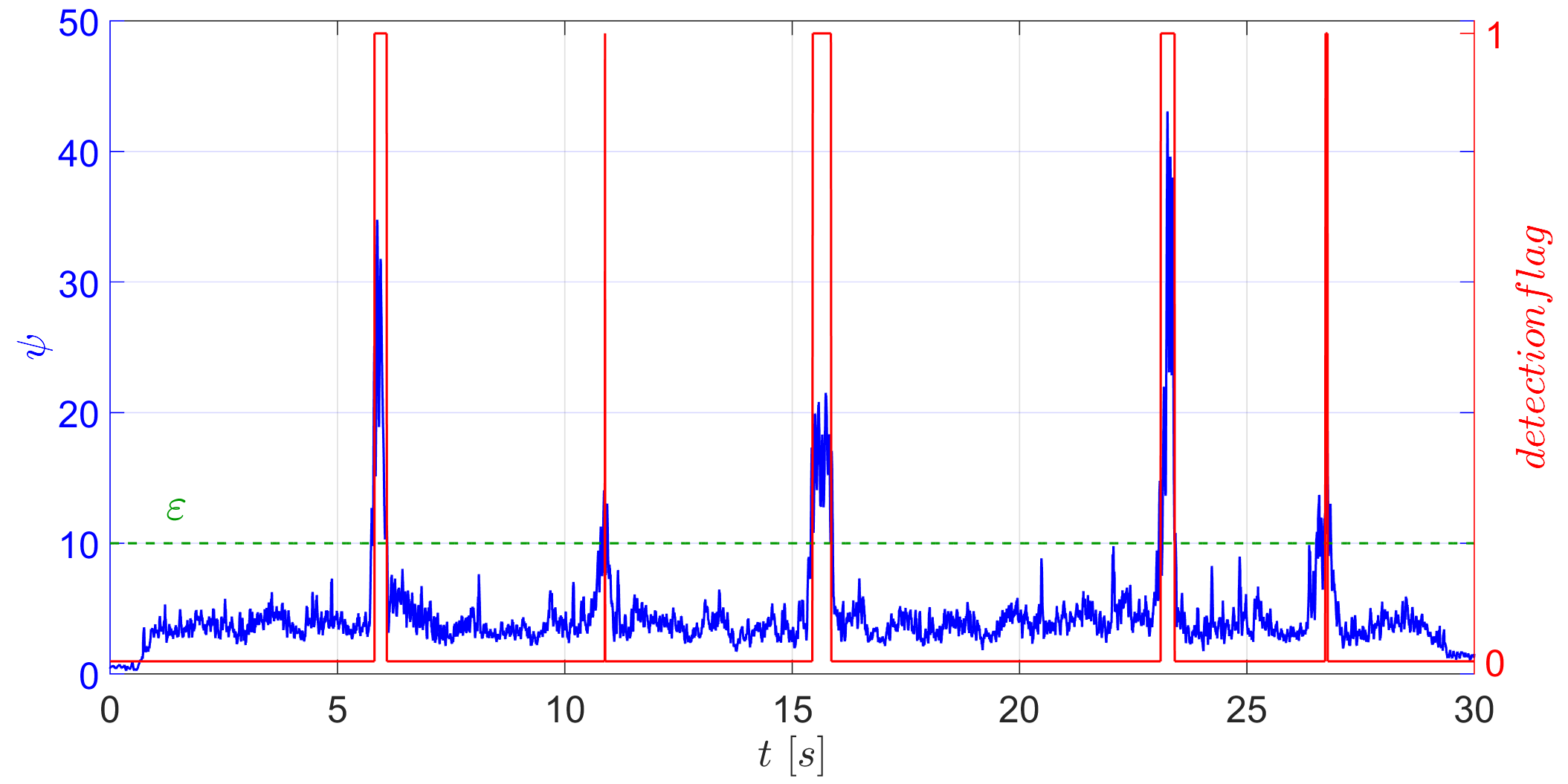}
  		\caption{Evolution over time of $\psi$. A detection flag (red line) is added to show when the heuristic detects deviations from the nominal behavior.}
  		\label{fig:epsilon}
	\end{figure}
{In order to avoid false positives, the value of $\psi$ was evaluated computing the moving average over a time window of duration $30\ ms$. A flag was then raised when such moving average was larger than the threshold $\varepsilon$: this flag is depicted with a red line in Fig.~\ref{fig:epsilon}. The average measured detection time was equal to $0.165\ s$. From the user's perspective, this is a sufficiently short amount of time, that allows to adapting the parameters before the user can actually feel the instability.}
From the results of the experiments, the heuristic was always able to detect the deviations from the nominal behavior, and thus the adaptation of parameters can be performed accordingly.

\subsection{Parameter adaptation}
We performed several experiments, in different scenarios, to validate the passivity framework for parameter adaptation. In particular, relative variations of $M_d(t)$ and $D_d(t)$ were defined according to different functional relationships. Since $M_d(t)$ and $D_d(t)$ can be freely modified, if condition \eqref{eq:cond1} or \eqref{eq:cond2} are satisfied, these functional relationships can be arbitrarily defined, based on the task that has to be accomplished.
We will hereafter show two representative examples: in the former, we consider constant damping and variable inertia. This represents the most problematic case, since variations on $D_d(t)$ do not negatively influence passivity, while variations on $M_d(t)$ could possibly break the passivity of the system, as shown in Section \ref{passivity}. In the latter, we aim at keeping a constant ratio between the inertia and the damping. This choice, as discussed in \cite{lecours}, aims at maintaining a similar dynamics of the system after the variation, which is more intuitive for the operator. 

\subsubsection{Passivity framework for parameter adaptation with constant damping and variable inertia}
Once deviations from the nominal behavior have been detected, then the admittance parameters have to be adapted for restoring the nominal behavior. In this experiment we implemented the strategy presented in Section \ref{passivity} and we investigated the differences between the conditions \eqref{eq:cond1} and \eqref{eq:cond2}. According to the theory developed in Section \ref{passivity}, the following design choices were made.
\setdefaultleftmargin{0.25cm}{}{}{}{}{}
\begin{itemize}
\item The energy thresholds have been selected as $\delta=0.1\ J$ and $\bar{T}=5\ J$. {The initial value for the state of the tank is set to $z(0)=2$, so that the initial energy contained in the tank, computed according to~\eqref{eq:tanken} results in \mbox{$T(0)=2\ J > \delta$}.} 
\item The bounds on the velocity are given by \mbox{$\dot{x}_M=\{\dot{x}_{M1}, \dot{x}_{M2}\}$}, 
where
\begin{eqnarray}
\dot{x}_{M1}=\{1.3, 1.5, 1.3\}\ \nicefrac{m}{s} \nonumber \quad
\dot{x}_{M2}=\{0.9, 0.9, 0.9\}\ \nicefrac{rad}{s} \nonumber
\end{eqnarray} 
\item The diagonal elements of the damping matrix $D_d$ are kept constant to $5\ \nicefrac{Ns}{m}$ and $0.5\ \nicefrac{Nms}{rad}$ for the translations and rotations, respectively, while the initial inertia is set to 
\mbox{$
M_d(0)=\diag\{2, 2, 2, 0.5, 0.5, 0.5\}\ kg
$}.
Since we want to show the increased flexibility of the tank approach with respect to the standard passivity condition, we will compare the results of the inertia variations as defined in \eqref{eq:newcond1} and \eqref{eq:newcond2}. In particular, {considering a variation of the inertia performed in the time interval $\left[t_i,t_f\right]$, integrating~\eqref{eq:newcond1} we define the following fixed step variation:}    
\begin{equation}
{M_d(t_f)-M_d(t_i)=2D_d \left(t_f - t_i\right) }
\label{eq:step}
\end{equation}
{In the experiments, we utilized $\left(t_f - t_i\right) = 3\ ms$.} 
{The increment in~\eqref{eq:step} is the}
 maximum increase allowed under the conservative condition \eqref{eq:cond1}. However, as shown in Section \ref{passivity}, the energy tank framework allows much more aggressive inertia variations if the inequality \eqref{eq:bench} is satisfied. In particular, under condition \eqref{eq:cond2}, {integrating~\eqref{eq:newcond2}, we obtain the following, less conservative, condition:} 
\begin{equation}
M_d(t_f)-M_d(t_i)=\dfrac{2\left(T(t_i)-\delta\right)}{{\Vert \dot{x}_M \Vert}^2}
\label{eq:step_tank}
\end{equation}
{It is worth noting that~\eqref{eq:step_tank} represents the maximum allowed inertia variation, based on the energy contained in the tank at time $t=t_i$. In practical cases, this value can be very large: thus, direct application of~\eqref{eq:step_tank} would lead to an excessively large inertia variation. For this reason, in the experiments we utilized the following upper-bound on the allowed inertia variation
	\begin{equation}
	M_d(t_f)-M_d(t_i) \leq \Delta M
	\label{eq:deltaM}
	\end{equation}
In the experiment $\scalemath{0.88}{\Delta M = \diag\{1.5, 1.5, 1.5, 0.15, 0.15, 0.15\}\ kg}$ was chosen empirically.}
\end{itemize}
\indent Figure \ref{fig:inertia_bench} shows the comparison between the inertia variations as defined in \eqref{eq:cond1} and \eqref{eq:cond2}. Furthermore, the figure shows the detection flag, which is raised when the heuristic detects a deviation from the nominal behavior. In correspondence of the detected deviations, the inertia varies according to \eqref{eq:step} (blue line) and \eqref{eq:step_tank} (green line). As it can be seen, the exploitation of the energy tank allows to reach higher values of the inertia, with respect to the standard passivity condition which is more conservative. 
%
%
\begin{figure}[tbp]
  \centering
  \includegraphics[width=\columnwidth]{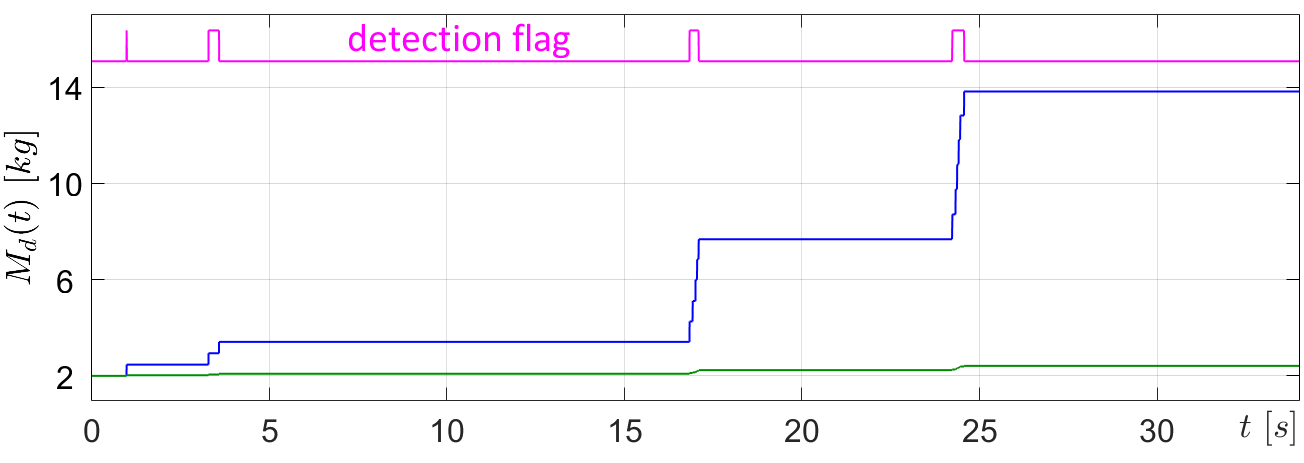} 
  \caption{Comparison between the inertia variation with the energy tank condition (blue line) and the standard passivity condition (green line).}
  \label{fig:inertia_bench}
\end{figure}

In Fig. \ref{fig:ten} the behavior of the tank energy $T$ is shown. At the beginning, the tank begins to store the energy dissipated by the system while the user moves the robot. Then, when deviations are detected, the energy decreases since the variation of the inertia requires energy to be implemented. 
\begin{figure}[tbp]
  \centering
  \includegraphics[width=\columnwidth]{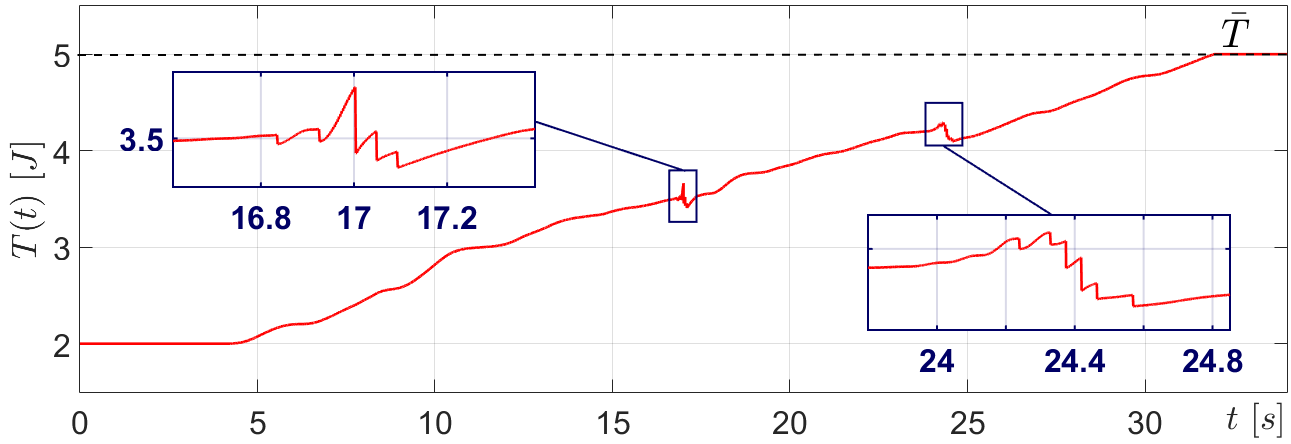} 
  \caption{Evolution over time of the energy level of the tank.}
  \label{fig:ten}
\end{figure}

\subsubsection{Passivity framework for parameter adaptation with constant inertia to damping ratio}
In this experiment, we kept a constant ratio between the inertia and the damping in order to maintain a similar dynamics of the system after the variation, which is more intuitive for the operator \cite{lecours}. The design choices are the same as the previous experiment except for the damping matrix that is variable and it is designed to guarantee a constant component-wise inertia to damping ratio. 
The initial damping is set to $15\ \nicefrac{Ns}{m}$ and $2\ \nicefrac{Nms}{rad}$ for the translations and rotations, while in \eqref{eq:deltaM} we choose $\scalemath{0.9}{\Delta M = \diag\{0.09, 0.09, 0.09, 0.012, 0.012, 0.012\}\ kg}$

Figure \ref{fig:all_tank} shows the results of the experiment of parameter adaptation with the energy tank framework. In particular, Fig. \ref{fig:pos_vel} shows the $\mathsf{x}$ component of the Cartesian position and velocity, while Fig. \ref{fig:force_tank} shows the forces measured by the F/T sensor. The user stiffens his/her arm at $t_1=5.6\ s$ and $t_2=11.1\ s$, an oscillating behavior starts to occur (yellow areas), but the adaptation of the parameters allows to stabilize the system $0.3-0.4\ s$ after the occurrence of the oscillations. From the user perspective, this is a sufficiently short amount of time, since the adaptation of the parameters is achieved before the user can actually feel the instability.   
\begin{figure}[tbp]
  \centering
    \subfigure[Cartesian position (solid green line) and velocity (dash blue line) along the $\mathsf{x}$-axis.]{\includegraphics[width=\columnwidth]{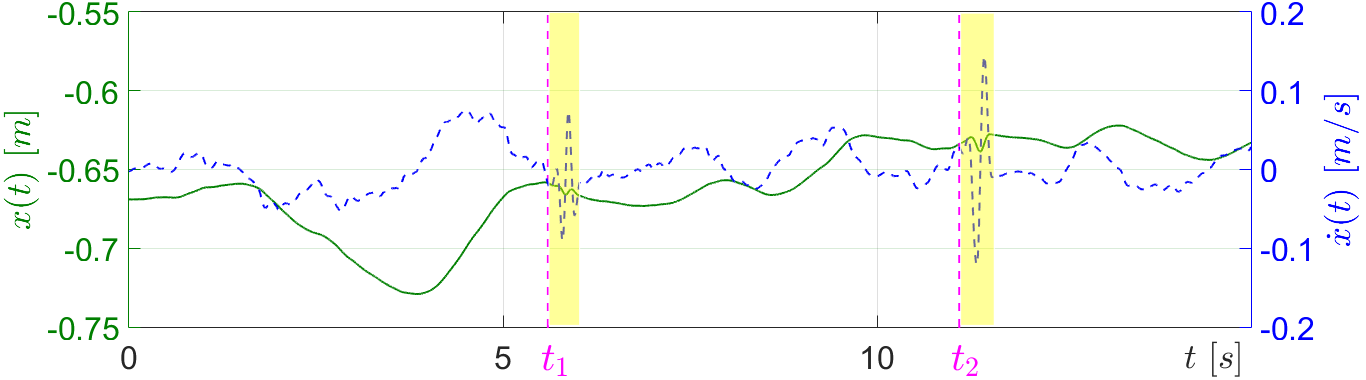} \label{fig:pos_vel}}
      \subfigure[Force along the $\mathsf{x}$-axis measured by the F/T sensor.]{\includegraphics[width=\columnwidth]{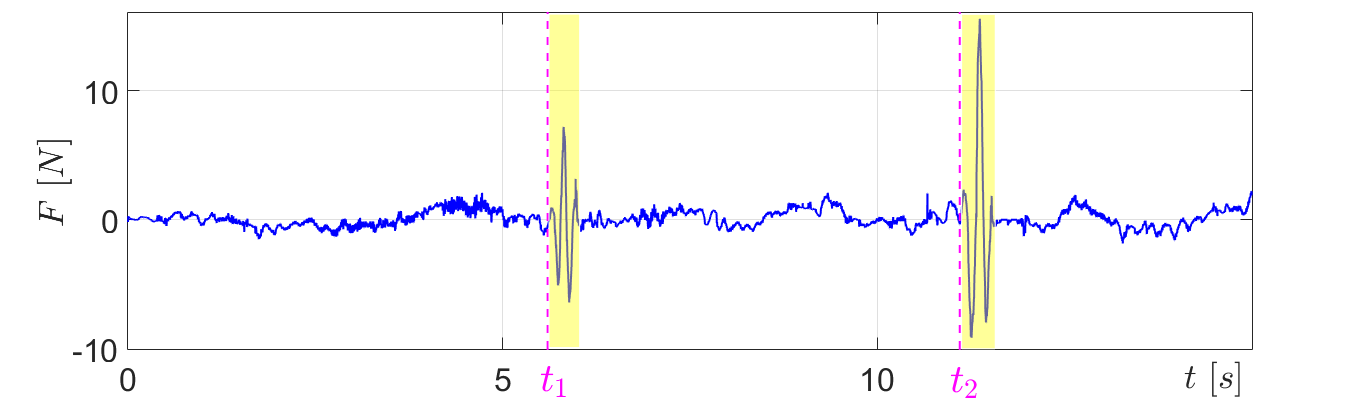} \label{fig:force_tank}}
  \caption{Parameter adaptation with the energy tank framework.}
  \label{fig:all_tank}
\end{figure}


\section{CONCLUSIONS}
\label{conclusion}

In this work, we introduced a heuristic that allowed to easily detect the deviations from the nominal behaviors and we presented a passivity-framework for adapting the parameters of the admittance control while maintaining the passivity of the overall system. \\Future works aim at improving the heuristic that we found, in order to render the detection independent from an application-dependent threshold. In addition, a forgetting-factor will be implemented in order to reduce inertia and damping factors when the operator relaxes his/her arm.

\bibliographystyle{IEEEtran}
\bibliography{biblio}

\end{document}